\documentclass[runningheads]{llncs}
\usepackage{graphicx}
\usepackage{amsmath,amssymb} 
\usepackage{color}
\usepackage{bbm} 
\usepackage{algpseudocode}
\begin{document}
\pagestyle{headings}
\mainmatter
\def\ECCV12SubNumber{1220}  

\title{Fast Planar Correlation Clustering for Image Segmentation} 

\titlerunning{Fast Planar Correlation Clustering}
\authorrunning{Yarkony, Ihler, Fowlkes}

\author{Julian Yarkony, Alexander Ihler, Charless C. Fowlkes\\
{\tt\{jyarkony,ihler,fowlkes\}@ics.uci.edu}}
\institute{Department of Computer Science, University of California, Irvine}

\maketitle

\begin{abstract}
We describe a new optimization scheme for finding high-quality clusterings in
planar graphs that uses weighted perfect matching as a subroutine. Our method provides
lower-bounds on the energy of the optimal correlation clustering that are
typically fast to compute and tight in practice. We demonstrate our algorithm
on the problem of image segmentation where this approach outperforms existing
global optimization techniques in minimizing the objective and is competitive
with the state of the art in producing high-quality segmentations.
\footnote{This is the extended version of a paper to appear at the 12th European Conference on Computer Vision (ECCV 2012)}
\end{abstract}

\section{Introduction}

We tackle the problem of generic image segmentation where the goal is to
partition the pixels of an image into sets corresponding to objects and 
surfaces in a scene. Cues for this task can come from both bottom-up (e.g.,
local edge contrast) and top-down (e.g., recognition of familiar objects). For
closed domains where top-down information is available, this problem can be
phrased in terms of labeling each pixel with one of several category labels or
perhaps ``background''.  There is a rapidly developing body of research in this 
area that integrates multiple cues such as the output of a bank of object
detectors into a single model, typically formulated as a Markov random
field over the pixel labels and some additional hidden variables~\cite{Boykov:PAMI01,BorensteinUllman:TPAMI08,winn2006layout,levin2009learning,Yang11,Gould,Ladicky}.

When top-down information is not available, it may still be quite valuable to
estimate image segments. Bottom-up segmentations provide candidate support for
novel objects and can simplify the processing of the scene to the problem of
understanding a small number of salient regions. Without a predefined set of
labels, it is natural to describe the segmentation task as a graph partitioning
problem in which pixels or superpixels have pairwise or higher order
similarities and the number of parts must be estimated.  There is a rich
history of applying graph partitioning techniques to image segmentation (e.g.,
\cite{N-Cuts,Wang:etc,ZhuSongShi:ICCV2007,Sharon_al:Nature06,cour:ncuts}). 

Here we consider the weighted correlation clustering objective which sums up
the edges cut by a proposed partitioning of the graph. Edges may have both
positive and negative weights. Correlation clustering is appealing since the
optimal number of segments emerges naturally as a function of the edge
weights rather than requiring an additional search over some model order
parameter.  Further, because the objective is linear in the edge weights, the
problem of learning can be approached using techniques from structured
prediction~\cite{Taskar}.

As with many non-trivial graph partitioning criteria, finding a minimum weight
correlation clustering is NP-hard for general
graphs~\cite{BansalBlumChawla}. Demaine et al. \cite{Demaine} provide results on
the hardness of approximation in general graphs by reduction to/from multiway
cut~\cite{Dahlhaus}. Recently, Bachrach et al. \cite{Bachrach} also showed that
correlation clustering is NP-hard in planar graphs by a reduction from planar
independent set.

Despite these difficulties, correlation clustering has seen a few recent
applications to the image segmentation problem.  Andres et al. \cite{Andres11}
define a model for image segmentation that scores segmentations based on the
sum of costs associated with each edge in the segmentation and optimize it
using an integer linear programming (ILP) branch-and-cut strategy. Kim et al.
\cite{Kim11} use a correlation clustering model for segmentation which includes
higher-order potentials or hyper-edges that define cost over sets of nodes
which they solve using linear programming (LP) relaxation techniques.

In this paper, we describe a new optimization strategy that specifically
exploits the planar structure of the image graph.  Our approach uses weighted
perfect matching to find candidate cuts in re-weighted versions of the
original graph and then combines these cuts into a final clustering. The
collection of cuts form constraints in a linear program 
that lower-bounds the energy of the true correlation clustering.  In practice
this lower-bound and the cost of the output clustering are almost always
equal, yielding a certificate of global optimality.  We compare this new
optimization scheme to existing approaches based on both standard LP relaxations
and ILP and find that our approach is substantially faster and provides tighter 
lower-bounds for a wide range of image segmentation problems.

\section{Correlation Clustering}

Correlation clustering is a clustering criteria based on pairwise
(dis)similarities.  Let $G = (V,E)$ be an undirected graph with edge
weights $\theta_e \in \mathbb{R}$ that specify the similarity or
dissimilarity on an edge $e = (i,j)$  between vertices $i$ and $j$.
Correlation clustering seeks a clustering of the vertices into disjoint sets
$V = V_1 \sqcup V_2 \sqcup V_3 \ldots$ that minimizes the total weight of edges
between clusters.  \footnote{This objective is equivalent (up to a constant)
with the minimum-disagreement or maximum-agreement objectives mentioned in the
literature~\cite{BansalBlumChawla,Demaine}.}

Let $X_e$ be a binary indicator variable specifying which edges are ``cut'' by
the partitioning.  $X_e = 0$ if edge $e = (u,v)$ is within a cluster (i.e.,
$u,v \in V_i$) and $X_e = 1$ if $e$ runs between two clusters (i.e., $u \in
V_i$, $v \in V_j$, $i \neq j$).  Let $\mathcal{C}$ indicate the configurations
of $X$ that correspond to valid partitionings of the vertices. We can describe
this succinctly by the set of triangle inequalities
\[
\mathcal{C} = \{X : X_{u,w} + X_{w,v} \geq X_{u,v} \quad \forall u,v,w \in V \}
\]
These constraints enforce transitivity of the clustering; if edge $(u,w)$ is
cut, then at least one of $(u,w)$ and $(w,v)$ must also be cut.

We can express the correlation clustering problem as:
\[
   CC^\star = \min_{X \in \mathcal{C}} \sum_{e \in E} \theta_e X_e
\]
and refer to $CC^\star$ as the cost or the energy of the optimal clustering.
Where appropriate, we will use $CC^\star(\theta)$ to indicate the dependence
of this optimum on the parameters.

Unlike other graph partitioning objectives (min-cut, normalized-cut, etc.) the
edge weights can be both positive and negative.  Furthermore, we do not specify
the number of segments {\it a priori} or place any constraints on their size.  Instead,
these arise naturally from the edge weights.  For example, if all the edge weights
are negative, each vertex will be placed in a separate cluster.  If all
the edge weights are positive, the optimal solution is to place all vertices in
a single cluster.  This means that $CC^\star$ is upper-bounded by 0 since
placing all the vertices in the same cluster is a valid partitioning with cost
0.

The correlation clustering objective appears very similar to standard pairwise
Markov Random Field (MRF) models for image labeling.  For
example, if we knew the optimal solution consisted of k clusters, we could
convert the problem into a $k$-state MRF without any unary terms. In the next
section we make this connection precise.

%
%

\section{Clusterings and Colorings}

Consider a partitioning of the graph represented by $X \in \mathcal{C}$.
We call this partitioning {\em $k$-colorable} if there is some labeling $L:V
\rightarrow {1,2,\ldots,k}$ of the vertices of the graph so that $X_{uv} = 1
\Leftrightarrow L(u) \neq L(v)$.  For every graph, there is a minimal number of colors
$\gamma(G)$, known as the {\em chromatic number} of $G$, that is sufficient to
represent all partitions.  Let $\mathcal{C}_k$ be the set of partitionings that
are representable by $k$ colors, then $\mathcal{C}_1 \subset \mathcal{C}_2
\subset \ldots \subset \mathcal{C}_{\gamma(G)} = \mathcal{C}$.  For example,
the four-color theorem \cite{AppelHaken} shows that any partition of a planar
graph can be represented by $k=4$ labels so $\mathcal{C} = \mathcal{C}_4$
for planar graphs.

This provides a useful alternative formulation of correlation clustering in
terms of vertex labels.  Let $L_v \in {1,\ldots,k}$ be a label variable for
vertex $v$.  Then we can define an equivalent optimization problem
\[
  CC_k^\star = \min_{L \in \{1,\ldots,k\}^N} \sum_{(u,v) \in E} \theta_{u,v} [L_u \neq L_v]
\]
To produce a partitioning from the labeling, simply take the collection of
connected components for the subgraph induced by each label in turn.  If $G$ is
$k$-colorable, then $CC_q^\star = CC^\star$ for any $q \geq k$. In general
$CC_q^\star \geq CC^\star$ for $q < k$ since the optimal partitioning of $G$ may
not be $q$-colorable.  The set of 2-colorable partitions are commonly referred
to as cuts of a graph.

Since planar graphs are 4-colorable, $CC^\star = CC_4^\star$.  We could tackle
the problem planar correlation clustering using the standard set of tools for
optimizing a 4-state MRF with mixed (attractive/repulsive) potentials. Since 
the combinatorial optimization is NP hard, such methods can only give approximate
solutions. Furthermore, many of them perform poorly on problems with no unary
potentials.  Since the energy function is symmetric with respect to
permutations of the labels, the true max-marginals for the node labels are
uninformative and one is forced to look at higher-order constraints.
For example, the lower-bound provided by TRW \cite{Wainwright05,Kolmogorov06}
is simply the sum of the negative edge weights in the graph.

One interesting exception is the case of planar binary labeling problems.
For planar graphs, the cost of the optimal binary labeling
$CC_2^\star$ can be computed by an efficient reduction to weighted perfect
matching in a suitably augmented planar dual of the graph $G$.  This idea was
first described in the statistical physics literature by Kasteleyn \cite{Kasteleyn2} and Fisher \cite{Fisher2} in the context of
computing the partition function of an Ising model. Recently,
this has been explored as a tool for finding MAP configurations of more general
MRFs that include an external
field~\cite{Globerson,SchraudolphNIPS,RPM,Batra,Yarkony11a}.

Since 2-colorable partitions are a subset of 4-colorable partitions, finding
the optimal 2-colorable partitioning does not necessarily give us the optimal
clustering of a planar graph. The space of 4-colorable partitions is larger
so in general $CC_2^\star \geq CC_4^\star$. However, the optimal 2-coloring still
provides some useful information about the optimal 4-colorable partition. 

\begin{proposition}
For any graph, $0 \geq CC_2^\star \geq {CC}_4^\star \geq \frac{3}{2} {CC}_2^\star$
so the cost of the optimal planar correlation clustering is bounded below by
$\frac{3}{2}$ the cost of the optimal 2-colorable partition.
\end{proposition}
\begin{proof}
For a partitioning described by some labeling $L$, let
\[
S(a,b) = \sum_{\substack{u: L_u = a \\ v: L_v = b}} \theta_{u,v}
\]
denote the sum of weights of edges between vertices labeled $a$ and those labeled $b$.
Take the 4-colorable partition whose cost is ${CC}_4^\star = \sum_{a<b}
S(a,b)$ and consider the 2-colorable partitions in which pairs of labels
from the 4-coloring are merged.  There are three such 2-colorings, each with
the following costs
\begin{align}
E_a = S(1,3) + S(1,4) + S(2,3) + S(2,4)\\
E_b = S(1,2) + S(1,4) + S(3,2) + S(3,4)\\
E_c = S(1,2) + S(1,3) + S(4,2) + S(4,3)
\end{align}
Summing these costs includes every possible $S$ term twice so we have
\begin{align}
CC_4^\star &= \frac{1}{2} (E_a + E_b + E_c)\\
           &\geq \frac{3}{2} \min \{E_a,E_b,E_c\}\\
           &\geq \frac{3}{2} {CC}_2^\star 
\end{align}
The first inequality follows since one of the three terms in the sum must
be at least as small as $1/3$ the total.  The second inequality follows since
none of these $2$-colorings can have lower cost than the optimum $2$-coloring.
The same approach can be used to relate any pair of $CC^\star_m$ and
$CC^\star_n$.
\end{proof}

\begin{corollary}
If ${CC}_2^\star= 0$ then ${CC}_4^\star=0$.
\end{corollary}

%
Since the 2-colorable clustering is only off by a constant factor and provides
a very efficient solution for finding approximate correlation clusterings, it 
seems a likely candidate for segmentation.  However, in practice, it performs
poorly for real image segmentation problems.  In natural images, T-junctions
where three different image segments come together are common, and such a junction
cannot be 2-colored! In the next section, we devise a tighter bound which uses
the 2-coloring as a subroutine.

\section{Lower-bounding Planar Correlation Clustering}
Dual-decomposition provides a very general framework for tackling difficult 
problems by splitting them into a collection of tractable sub-problems which are
solved independently subject to the constraint that they agree on their
solutions.  This constraint is enforced in a soft way using Lagrange
multipliers, which results in a dual solution that lower-bounds the original
minimization problem.  Decomposition techniques have been studied in the
optimization community for decades.  Dual-decomposition was used by Wainwright
et al.~\cite{Wainwright05} to derive algorithms for inference in graphical
models and has become increasingly popular in the computer vision literature
recently due to its flexibility~\cite{Komodakis07}.

We consider bounding the planar correlation clustering by a decomposition into
two sub-problems, an easier partitioning problem and an independent edge
problem that does not enforce the clustering constraints.  To make the
partitioning problem tractable, we impose a constraint on the decomposition so
that the cost of the optimal clustering can be computed. Recall our notation
that $CC^\star(\lambda)$ is the optimal correlation clustering cost associated
with edge weights $\lambda$.  Let the set $\Omega = \{\lambda :
CC^\star(\lambda)=0\}$ be those those edges weights for which the optimal
clustering has zero cost.  We can then write the following decomposition bound:

\begin{align}
   CC^\star &= \min_{X \in \mathcal{C}_4} \sum_{e \in E} \theta_e X_e \\
     &= \max_\lambda \left(\min_{\hat{X} \in \mathcal{C}_4} \sum_{e \in E} \lambda_e \hat{X}_e \right) + \left( \min_{\bar{X}} \sum_{e \in E} (\theta_e - \lambda_e) \bar{X}_e \right) \label{eqn:line1} \\
     &\geq \max_{\lambda \in \Omega} \left(\min_{\hat{X} \in \mathcal{C}_4} \sum_{e \in E} \lambda_e \hat{X}_e \right) + \left( \min_{\bar{X}} \sum_{e \in E} (\theta_e - \lambda_e) \bar{X}_e \right) \label{eqn:line2}  
\end{align}
\begin{align}
     &= \max_{\lambda \in \Omega} \left( \min_{\bar{X}} \sum_{e \in E} (\theta_e - \lambda_e) \bar{X}_e \right) \label{eqn:line3}\\
     &= \max_{\lambda \in \Omega} \sum_{e \in E} \min\{(\theta_e - \lambda_e),0\} \label{eqn:line4}
\end{align}

In equation (\ref{eqn:line1}) we have decomposed the original edge weights
$\theta$ across two sub-problems.  The first is a correlation clustering
problem (identical in form to our original problem) while the second one
independently optimizes over all the edges (with no constraints on $X$). 
For any choice of $\lambda$, these two objectives sum up to the original 
problem.  Since the configurations ($\hat{X}$,$\bar{X}$) in each sub-problem
are optimized independently, the sum of their energies may produce a lower-bound for
arbitrary $\lambda$ but the bound can be made tight
(setting $\lambda_e=\theta_e$ recovers the original objective).  In equation
(\ref{eqn:line2}), we restrict the domain of $\lambda$ to those settings for
which the clustering sub-problem has an optimum of zero.  The inequality arises
since we are maximizing the bound over a more restrictive set.  Finally, in
equation (\ref{eqn:line3}) we have simplified the expression since the
constraint on $\lambda$ entails that the first term is exactly zero and
$\bar{X}$ can be optimized independently for each edge.

\section{Bound Optimization using Linear Programming}
Lagrangian relaxation approaches typically use projected sub-gradient ascent or
other non-smooth optimization techniques to tackle objectives like that shown
in equation (\ref{eqn:line2}).  Here it is difficult to compute the required (sub)gradient
information since, for a given setting of $\lambda$, there isn't an obvious way
to recover the full set of optimizing solutions for $\hat{X}$ beyond the trivial
solution $\hat{X_e}=0$.  The constraint set $\Omega$ also appears quite
complicated. However, we do have an efficient method for testing membership in $\Omega$.  By our earlier proposition,
\begin{align}
\Omega &= \{\lambda : CC^\star(\lambda)=0\} = \{\lambda : CC^\star_2(\lambda)=0\} \\
       &= \{\lambda : \sum_{e} \lambda_e X_e \geq 0 \quad \forall X \in \mathcal{C}_2\}
\end{align}
This expression highlights that $\Omega$ is a polytope defined by a set of linear
inequalities. For a given $\lambda$, we can test membership and, if $\lambda \not\in
\Omega$, produce a violated constraint described by a negative
weight 2-colorable clustering.  This provides a method to solve equation
(\ref{eqn:line4}) using cutting planes to successively
approximate the constraint set $\Omega$.


We say that an edge $e$ is {\em constrained} by $\Omega$
for a given setting of $\lambda$ if there exists some cut in $X \in
\mathcal{C}_2$ with $X_e=1$ and for which $\sum_{e} \lambda_e X_e = 0$.  If an
edge $e$ is unconstrained, then we can decrease $\lambda_e$ and thereby
increase the bound until it becomes constrained or until it is no longer cut in
the independent edge problem.

To simplify the bound optimization, we first consider some additional
constraints on $\lambda$.  When maximizing the bound over $\lambda_e$, it is
always the case that there is an optimal $\lambda_e \geq \theta_e$.  Choosing
$\lambda_e < \theta_e$ gives the edge $e$ positive weight in the independent
edge problem so it does nothing to increase the objective.  Further, any amount
by which $\lambda_e$ is less than $\theta_e$ can only make the constraints
$\Omega$ more difficult to satisfy.  Therefore, we are free to only consider
$\lambda$ for which $(\theta_e-\lambda_e)\leq 0$ without impacting the final
bound. This simplifies the expression for the objective in equation
($\ref{eqn:line4}$) by removing the $\min$.

We also impose upper bounds on $\lambda$.  For edges with $\theta_e < 0$ we
add the constraint that $\lambda_e \leq 0$.  For edges with $\theta_e \geq
0$ we impose the constraint that $\lambda_e = \theta_e$. These constraints are
sensible in that they are coordinate-wise optimal (e.g., increasing $\lambda_e$
above $\theta_e \geq 0$ decreases the bound). In Appendix B, we show that any
optimal $\lambda$ can be deformed to one which satisfies these additional
constraints without loosening the bound.  In practice these additional
constraints make the bound optimization far more efficient.

We can now write the bound optimization problem with these additional constraints
explicitly as standard linear program:
\begin{align}
&\max_{\lambda} \sum_e (\theta_e - \lambda_e) \label{LP}\\
&\text{s.t.} \quad \theta_e \leq \lambda_e \leq \max\{0,\theta_e\} \notag \\
&\quad \quad \sum_{e} \lambda_e X_e \geq 0 \quad \quad \forall X \in \mathcal{C}_2 \notag
\end{align}

This LP has an exponential number of constraints, one for every possible
2-colorable partition $X$. To solve this LP
efficiently, we use a cutting plane approach to successively add violated
constraints to a collection.  Our final algorithm for bound optimization is
given in Figure \ref{figLP}. 

In our actual implementation, we perform one additional step.  Each new
constraint $X$ may partition the graph into multiple components.  We break the
cut $X$ up into the set of basic cuts, each of which isolates a component.  We add this
collection of constraints as a batch.  With this modification, we find that in
practice very few batches of constraints (typically 5-10) are necessary in
order to produce a solution to the full linear program.

\begin{figure}[t]
\fbox{
\begin{tabular}{cc}
\begin{minipage}{0.45\textwidth}
\mbox{\bf Lower-bound optimization}
\begin{algorithmic}
  \State $\mathcal{P} = \emptyset$
  \While {$CC^\star_2(\lambda) < 0$}
     \State $X = \arg\min_{X \in \mathcal{C}_2} \sum_{e \in E} \lambda_e X_e$
     \State $\mathcal{P} = \mathcal{P} \cup X$
     \State Solve (\ref{LP}) with partial constraint set $\mathcal{P} \subset \mathcal{C}_2$
  \EndWhile
\end{algorithmic}
\end{minipage}
&
\begin{minipage}{0.5\textwidth}
\mbox{\bf Upper-bound decoding}
\begin{algorithmic}
  \State $S = 0$
  \State $\mathcal{E} = \{e : \theta_e - \lambda_e < 0\}$
  \For{$e \in \mathcal{E}$}
     \State $X = \arg\min_{X \in \mathcal{C}_2, {X_e = 1}} \sum_{f \in E} \lambda_f X_f$
     \State $S' = \max(S,X)$
     \If {$ \sum_e \theta_e S'_e \leq \sum_e \theta_e S_e $}
        \State S = S'
        \State $\forall e : X_e=1, \quad \lambda_e = 0$
     \EndIf
  \EndFor
\end{algorithmic} 
\end{minipage} \\
\end{tabular}}
\caption{(left) Cutting plane algorithm for computing the optimal lower-bound
by successively adding constraints. (right) Upper-bound decoding by recursive
partitioning\label{figLP}}
\end{figure}

\section{Decoding upper-bounds}

\subsection{Recursive Bipartitioning}
Once we have optimized the lower-bound, we would like to find a corresponding
low-cost clustering.  
%
%
In general, there will be some edges for which $(\theta_e-\lambda_e)<0$.
In order for the bound to be tight, we need to find a clustering in which 
these edges are cut. As noted in the previous section, every such ``must cut''
edge is constrained by some cut $\hat{X}$ that includes that edge.  Although
none of the individual minimal cuts $\hat{X}$ may agree with all of the ``must
cut'' edges in the independent sub-problem (second term in equation
(\ref{eqn:line2})), there is some minimal cut that agrees with each one.  

Motivated by this intuition, one can use the following decoding technique.
Start with the original clustering sub-problem which has edge weights $\lambda$.
Choose an ordering of those edges $e$ for which $(\theta_e-\lambda_e)<0$.  For
each of these ``must cut'' edges in turn, find a zero weight cut $X \in
\mathcal{C}_2$ for the clustering subproblem and add it to the final
partitioning as long as it decreases the original objective. Remove these cut
edges from the graph and continue on with the next edge.  
The pseudo-code is displayed in Figure \ref{figLP}.

\subsection{Dual LP Rounding}

An alternative approach is to consider the dual LP to equation (\ref{LP}).
Let $C$ be a
matrix whose rows contain the indicator vectors for cuts $X \in {\mathcal C}$.
Define the convex cone ${\mathcal C}_2^\triangle = \{C^T\alpha, \alpha\geq 0\}$ which
is known as the ``cut cone''~\cite{Deza}.  
It is straightforward to see that the set of valid partitions lives inside the
cut cone (${\mathcal C} \subset {\mathcal C}_2^\triangle$).  Given a valid
partition indicator vector $X$, we can write it as a linear
combination of cuts, where each cut isolates an individual segment and the cuts
are assigned a weight of $\alpha_i = 0.5$.  

The dual LP to our lower-bound is given by
\begin{align}
&\min_{z} \theta^T z - \min(\theta^T,0)\max(z-1,0) \label{eqn:dualLP1}\\
&\text{s.t.} \quad z \in {\mathcal C}_2^\triangle \notag
\end{align}
The first term in the objective is exactly our original correlation clustering
objective where the binary indicator $X$ has been replaced by real valued $z$.
The second term in the objective arises from the upper-bound constraints
imposed on $\lambda$ and effectively cancels out
the benefit of cutting any negative weight edge by an amount of more than one
(see Appendix A).
To compute a solution to the dual, we solve (\ref{eqn:dualLP1}) using a matrix $C$ that
contains only those cut vectors in ${\mathcal P}$ produced during the
lower-bound optimization.  The resulting solution vector $z$ is thresholded to produce
a final segmentation.

Taking the dual of equation (\ref{LP}) without the upper-bound constraints
on $\lambda$ yields a nice interpretation
of our algorithm as a convex relaxation of the original discrete clustering
problem in which the convex hull of ${\mathcal C}$ is approximated by the
intersection of the cut cone and the unit cube:
\begin{align}
&\min_{z} \theta^T z &\text{s.t.} \quad z \in {\mathcal C}_2^\triangle, \quad 0 \leq z \leq 1  \label{eqn:dualLP2}
\end{align}
Since the cut cone for planar graphs can be described with a polynomial number
of constraints~\cite{Barahona83}, one could directly solve the dual LP in equation
(\ref{eqn:dualLP2}).  Our bound optimization gains considerable efficiency by
not using the full set of cuts $C$.  Instead, a small number of cutting planes
(in the original LP) provides a delayed column generation scheme for solving
the dual LP. 

When only using a subset of cuts, the second term
in (\ref{eqn:dualLP1}) and the corresponding constraints in the primal LP are
necessary since the bound can be tight without the optimal partition
vector living in the subspace of the cut cone described by ${\mathcal P}$.
Allowing solutions with $z>1$ lets us access a larger set partitions without
increasing the dimensionality of the subspace.

\begin{figure}[t]
\centering

\includegraphics[width=2.0in,trim=1.5in 6in 1in 3in]{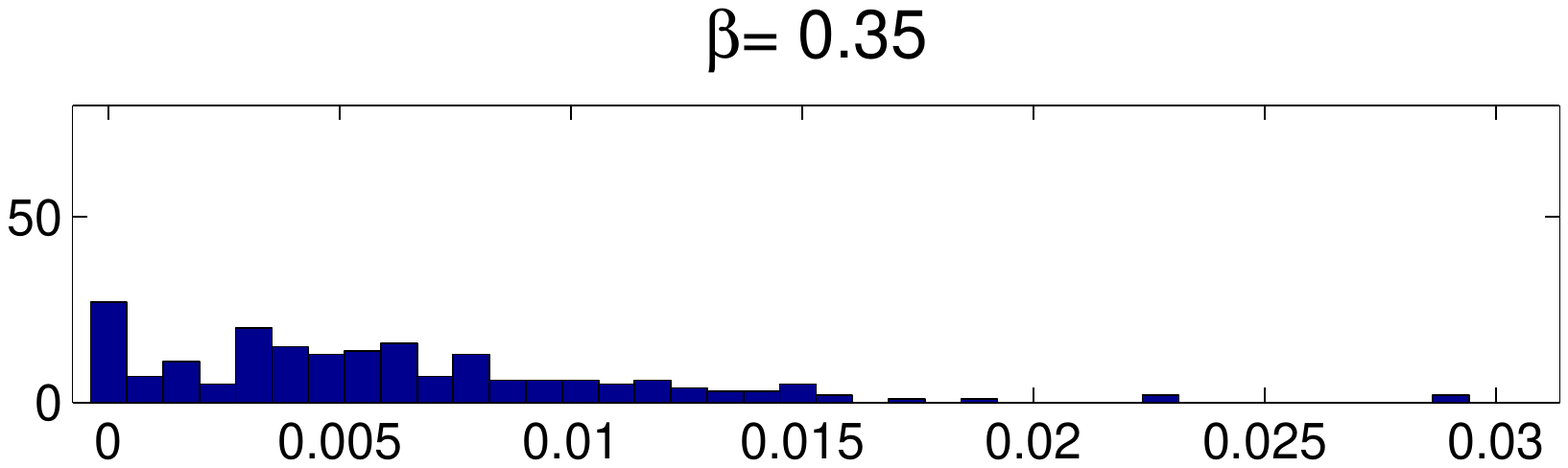}
\includegraphics[width=2.0in,trim=1.0in 6in 1.5in 3in]{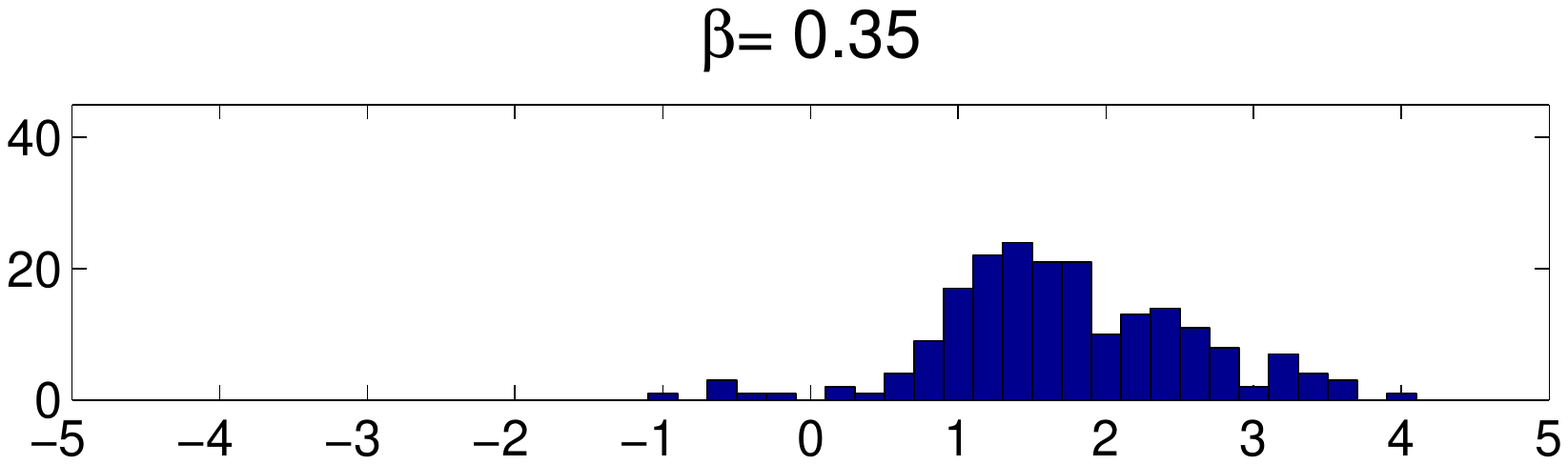}
\\
\includegraphics[width=2.0in,trim=1.5in 6in 1in 3in]{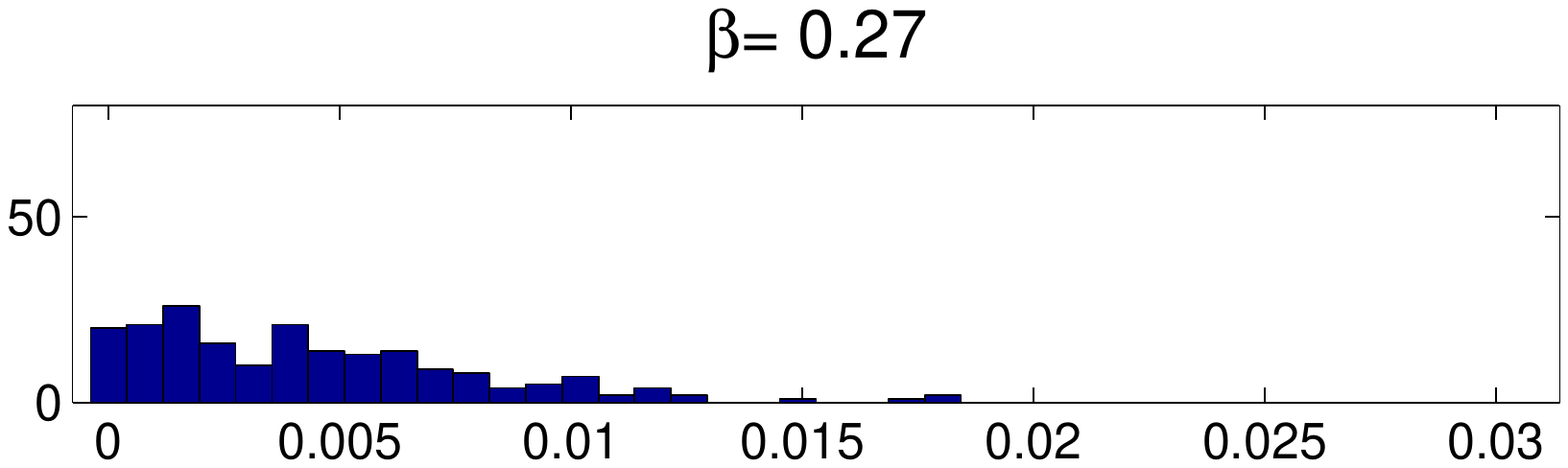}
\includegraphics[width=2.0in,trim=1.0in 6in 1.5in 3in]{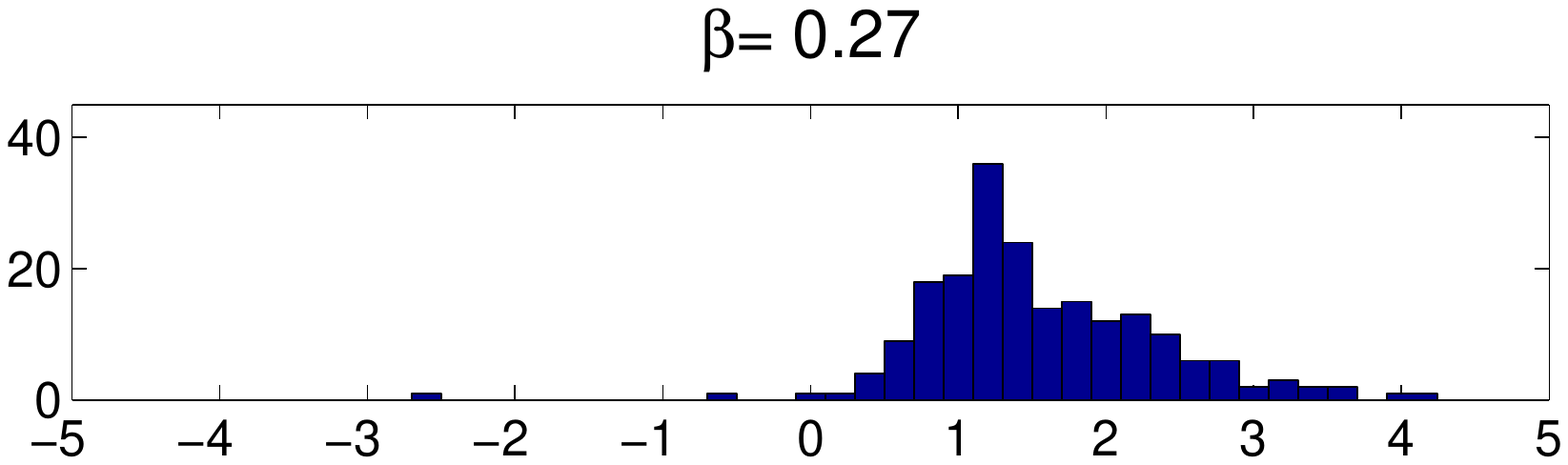}
\\
\includegraphics[width=2.0in,trim=1.5in 6in 1in 3in]{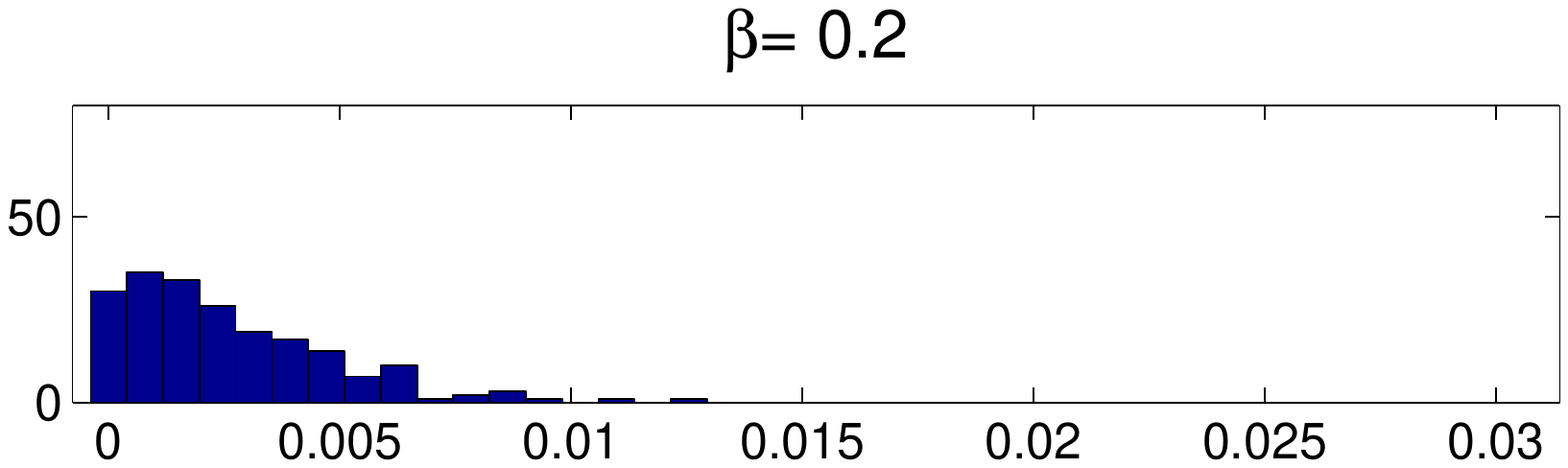}
\includegraphics[width=2.0in,trim=1.0in 6in 1.5in 3in]{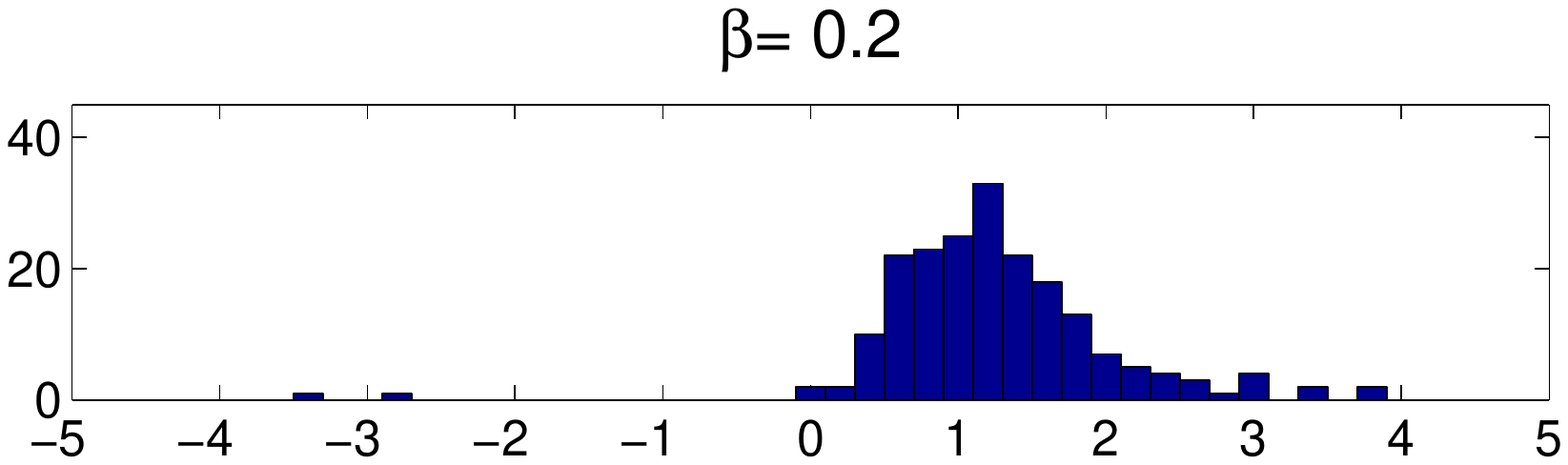}
\\
\includegraphics[width=2.0in,trim=1.5in 4.5in 1in 3in]{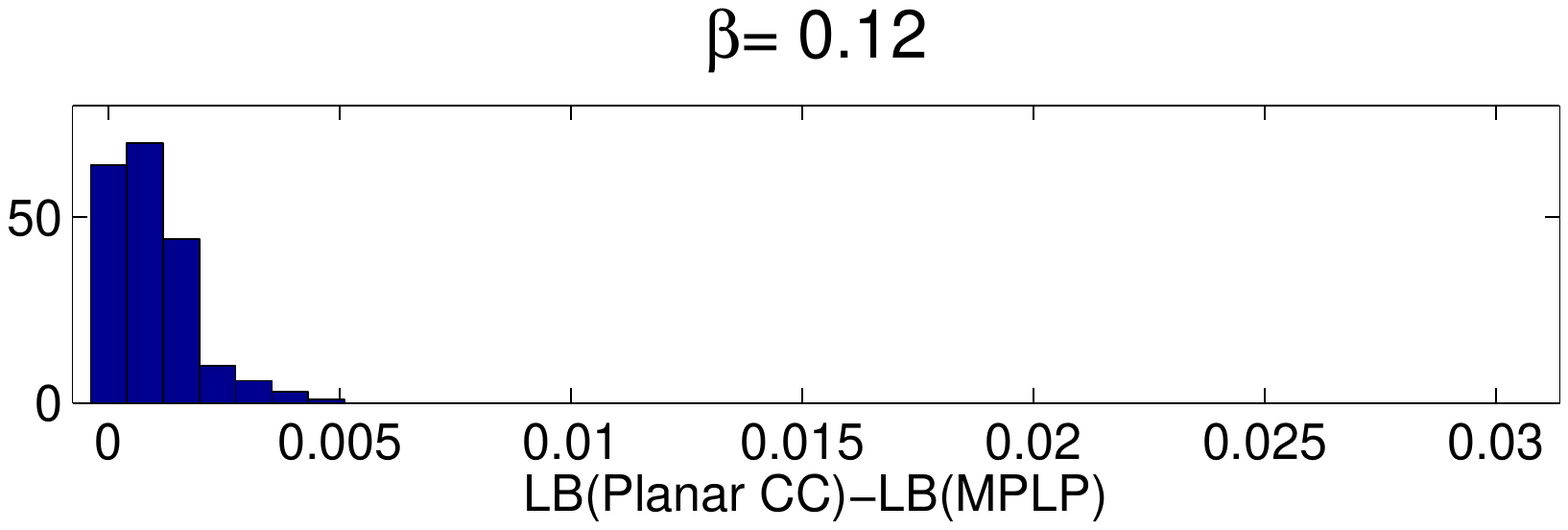}
\includegraphics[width=2.0in,trim=1.0in 4.5in 1.5in 3in]{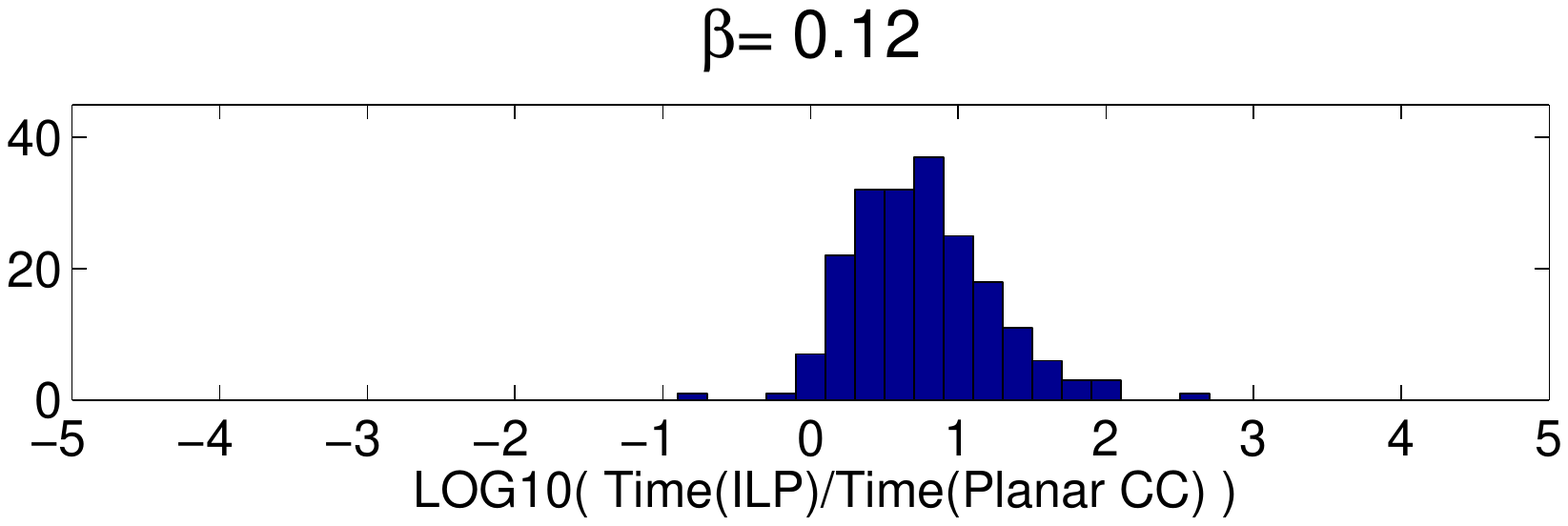}
\\

\caption{Comparison of bound optimization on image segmentation problems.  Each
graph shows the distribution results over 200 problem instances at four
different threshold settings ranging from coarse ($\beta=0.35$) to fine
($\beta=0.12$). The left column shows the difference in the lower-bounds
returned by the PlanarCC bound and MPLP using the set of face cycle
constraints. Our code returned tight bounds in all but one instance while the
LP relaxation typically gave looser bounds.  The right column shows the running
times of our approach compared to the ILP branch-and-cut advocated by
\cite{Andres11}.  Here we plot relative speedup factors on a logarithmic scale.
We find that the PlanarCC bound computation and decoding produces the same
global optima as the ILP approach much faster.  \label{fig:comparison0}}
\end{figure}

\section{Experiments}

We demonstrate the performance of our algorithm on correlation clustering
problem instances from the Berkeley Segmentation Data set~\cite{MFTM01,Arbelaez11}.
Our clustering problem is defined on the superpixel graph given by performing
the oriented watershed transform (owt) on the output of the ``generalized
probability of boundary'' (gPb) boundary detector output as proposed
by~\cite{Arbelaez11}.  Each pair of superpixels that are adjacent in the image
are connected by an edge whose weight is given by
\[
  \theta_e = \log\left(\frac{1-gPb_e}{gPb_e}\right) + \beta
\]
where $gPb_e$ is the average $gPb$ along the edge and $\beta$ is a threshold
parameter that modulates the number of segments in the optimal clustering. Large
$\beta$ results in more positive edges and hence coarser optimal segmentations.
To compare different optimizers, we use $4$ different settings of
$\beta=\{0.35,0.27,0.20,0.12\}$ which produces segmentation outputs that cover
a range of granularities. Parameters $\theta$ were rounded to 5 decimal places in
order to simplify tests of convergence.

We implemented our optimization using the BlossomV minimum weight perfect
matching code~\cite{Kolmogorov09,SchraudolphTR} and IBM's CPLEX solver to
optimize the lower-bound. We used a tolerance of $-10^{-6}$ as a stopping
criterion for adding additional
constraints to the lower-bound LP.  We found that both decoding schemes 
work well. In the experiments described here, we computed up to ten upper bounds using the
recursive bipartitioning procedure, each run using a random order for adding
contours.  This process terminated early if the lower- and upper-bounds were
equal.

For a baseline lower-bounding scheme, we used max-product linear programming
(MPLP)~\cite{Sontag08} which efficiently solves an LP relaxation of the original
clustering objective. To represent the set of clusterings in terms of node
labels, each superpixel takes on one of 4 states and pairwise potentials encode
the boundary strength between neighboring superpixels. As mentioned before, the
standard edge-based relaxation is uninformative when unary potentials are
absent so we include the set of cycle constraints given by collection of cycles
that bound the planar faces of the superpixel graph.  This is not sufficient to
enforce consistency over all cycles in the graph but is a natural choice
commonly used in the literature. In our experiments, we used a fast, in-house
implementation of MPLP.

We also implemented the branch-and-cut ILP technique proposed by \cite{Andres11}
using the CPLEX ILP solver. This approach finds an integral solution of the
correlation clustering objective, removes cut edges specified in the solution
and then produces a partition by finding connected components of the resulting
graph. It then searches for inconsistent edges, namely cut edges that connect
two nodes that lie within the same connected component. If any such edges are found,
a constraint is added to enforce consistency of that edge and the ILP is re-solved.

\subsection{Bound optimization experiments}

Figure \ref{fig:comparison0} shows a comparison of the lower-bounds generated
by MPLP compared to that generated by PlanarCC. We found that the time needed
for MPLP to solve each problem is comparable to that of PlanarCC.  However the
differences in the lower-bound are significant. With only the set of face cycles,
MPLP is seldom able to produce a tight lower-bound.  In contrast, the PlanarCC
approach typically gives tight bounds with only 5-10 batches of cut constraints.

We found that the upper-bounds (solutions) generated by ILP and PlanarCC are
very similar and very close to optimal so we compare the time consumed by each
algorithm as a function of $\beta$.  In Figure \ref{fig:comparison0}(a) we show
histogram of the comparative run times, $\log_{10}(T_{ILP}/T_{PlanarCC})$.

\begin{figure}[t]
\centering
\includegraphics[width=3in,trim=1in  1in 1in 1in]{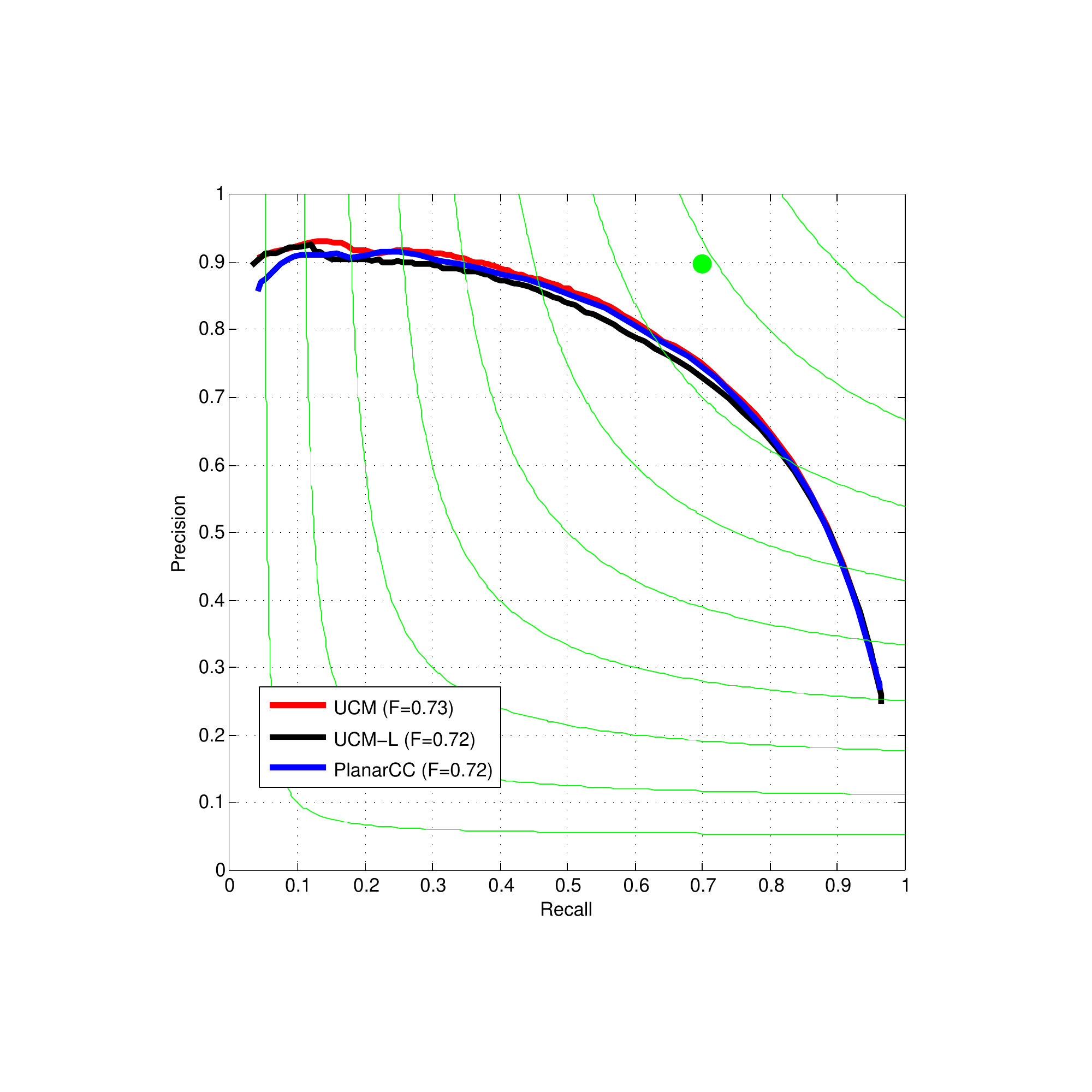}
\caption{Evaluation on the BSDS500 segmentation boundary benchmark. We compare
segmentation performance to the state of the art technique (gPb+owt+UCM)
proposed by \cite{Arbelaez11}. We use the same set of
superpixels and contour cues derived from (gPb+owt).  We compare to two
different variants of the UCM algorithm based on region merging. UCM
performs a length weighted average of the gPb along contours after each merge
while UCM-L performs a uniform average. We find that the globally optimal
correlation clustering returned by our algorithm performs slightly better than
the uniform averaging version of UCM but the length-weighted UCM gives better final
performance.
}
\label{fig:comparison1}
\end{figure}

Note that the relative performance of PlanarCC improves as we move from a high
detail segmentation $\beta=0.12$ to a coarse segmentation $\beta=0.35$.  For
coarse segmentations (large $\beta$) the optimal solution contains many long
contours and PlanarCC performs well relative to ILP, whereas for detailed
segmentations ILP will tend to do more favorably. For example, in the limit
where all the edges have negative weight, the ILP approach or LP relaxation
gives the correct answer (cut all edges) without the need for any constraints.
However on average, we find that the PlanarCC approach performs favorably across
a range of useful thresholds on the BSDS images, giving speedups that range from
10 to 1000x.


\subsection{Segmentation performance}

We benchmark the quality of the segmentations produced by correlation
clustering for a range of thresholds $\beta$ on the BSDS500 test set. We use
the same superpixels and local cues as the top performing gPb+owt+UCM algorithm
of Arbelaez et al.~\cite{Arbelaez11}. Figure \ref{fig:comparison1} shows the
benchmark results of our algorithm and two variants of the UCM algorithm.
As visible in the figure, our algorithm performs comparably to UCM and performs
slightly better than the results of \cite{Andres11} who report an F-measure of
$0.70$.

The UCM algorithm is a region merging algorithm that successively merges the
two superpixels that have the lowest energy edge between them.  Since this
algorithm is ``greedy'' with respect to the clustering objective,  we would
expect that it would occasionally merge two segments due to some small break in
the contour contrast, a fate that our global optimization approach could avoid. However,
as is clear from Figure \ref{fig:comparison1}, in practice the greedy nature of
UCM does not seem to significantly hurt overall performance.

One explanation is that the UCM algorithm modifies the edge costs as it
proceeds.  After each merging step, any new contours that have been formed are
re-assigned the average of the underlying $gPb$.  Our global clustering
objective cannot capture this length weighted averaging.  Figure
\ref{fig:comparison1} shows performance of the UCM algorithm with length-weighted
averaging (UCM) and simple averaging (UCM-L). While our approach outperforms
the non-length weighted version, the differences are not substantial.

Another possible explanation is that the greedy merging is truly successful in
optimizing the correlation clustering objective.  Figure \ref{fig:comparison2}
shows that this is not the case -- while there is usually some UCM threshold that provides a
segmentation with a fairly low-cost clustering, it is still suboptimal compared
to the solutions returned by PlanarCC. This suggests learning an optimal
cue combination via structured prediction may improve performance.

Finally, it is worth noting that the boundary detection benchmark does not
provide strong penalties for small leaks between two segments when the total
number of boundary pixels involved is small.  We found that on the region based
benchmarks, PlanarCC did outperform UCM slightly when the optimal segmentation
threshold was chosen on a per-image basis (GT Covering OIS 0.65 versus 0.64 for
UCM).  We expect these differences may become more apparent in an application
where the local boundary signal is noisier (e.g., biological imaging) or when
there is a greater cost for under-segmentation.

\section{Conclusion}
\begin{figure}[t]
\centering
\includegraphics[width=2.0in,trim=1.5in 6in 1in 3in]{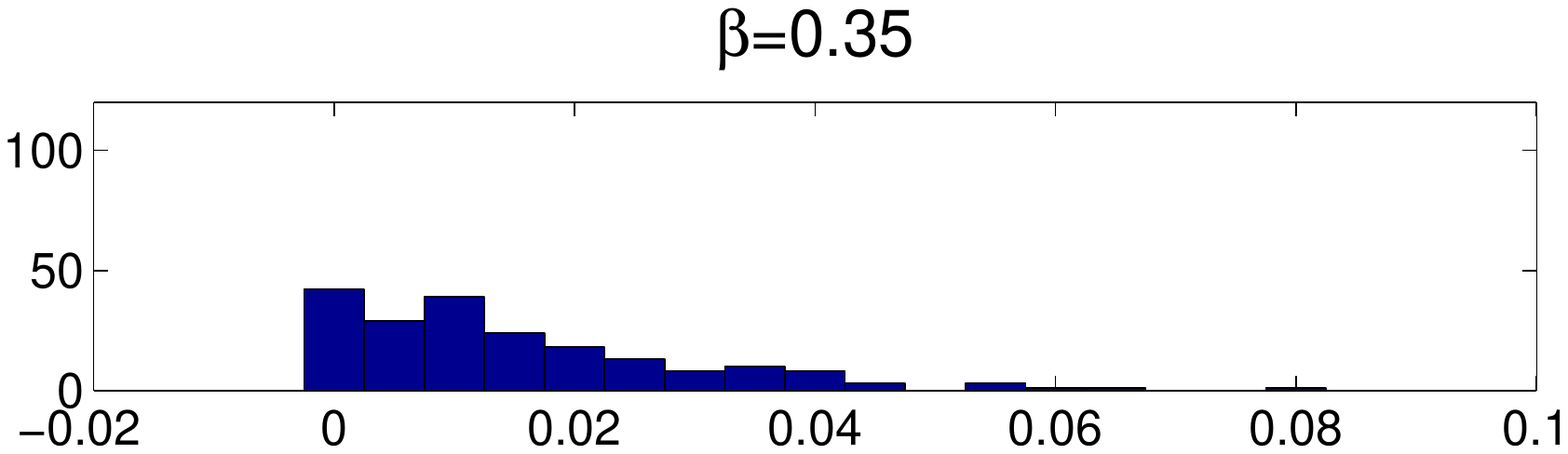}
\includegraphics[width=2.0in,trim=1.0in 6in 1.5in 3in]{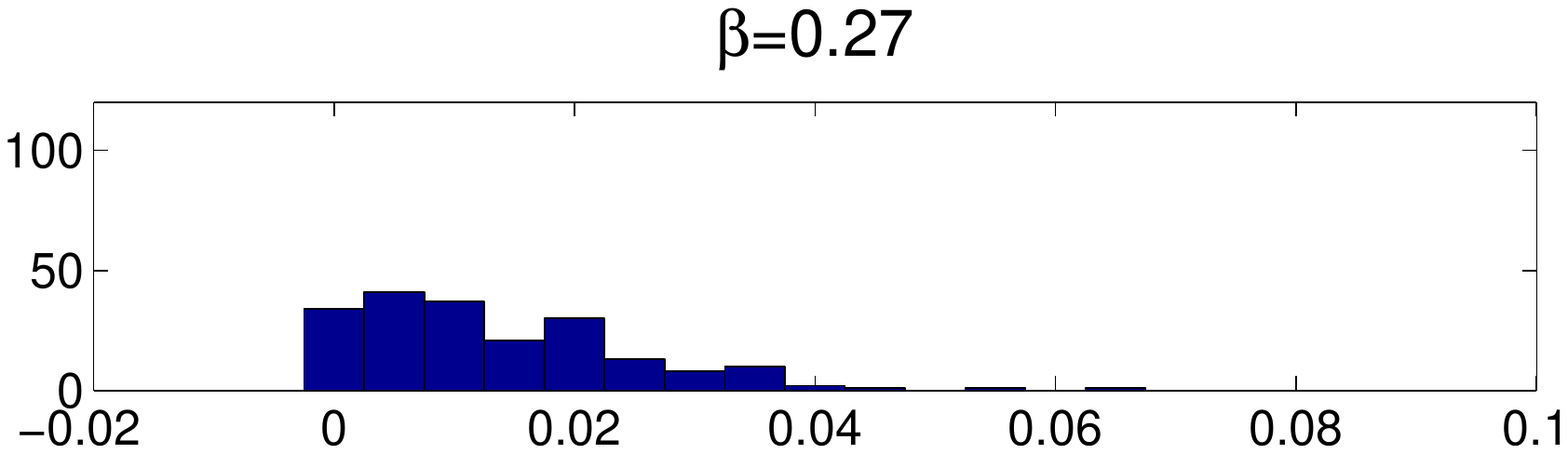}
\\
\includegraphics[width=2.0in,trim=1.5in 4.5in 1in 3in]{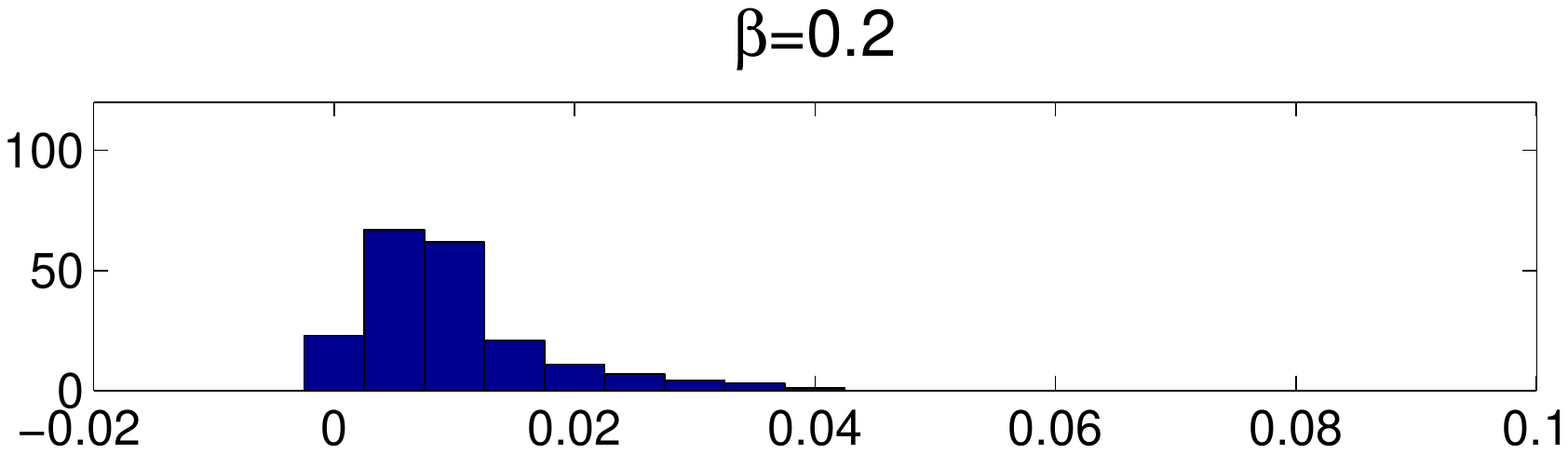}
\includegraphics[width=2.0in,trim=1.0in 4.5in 1.5in 3in]{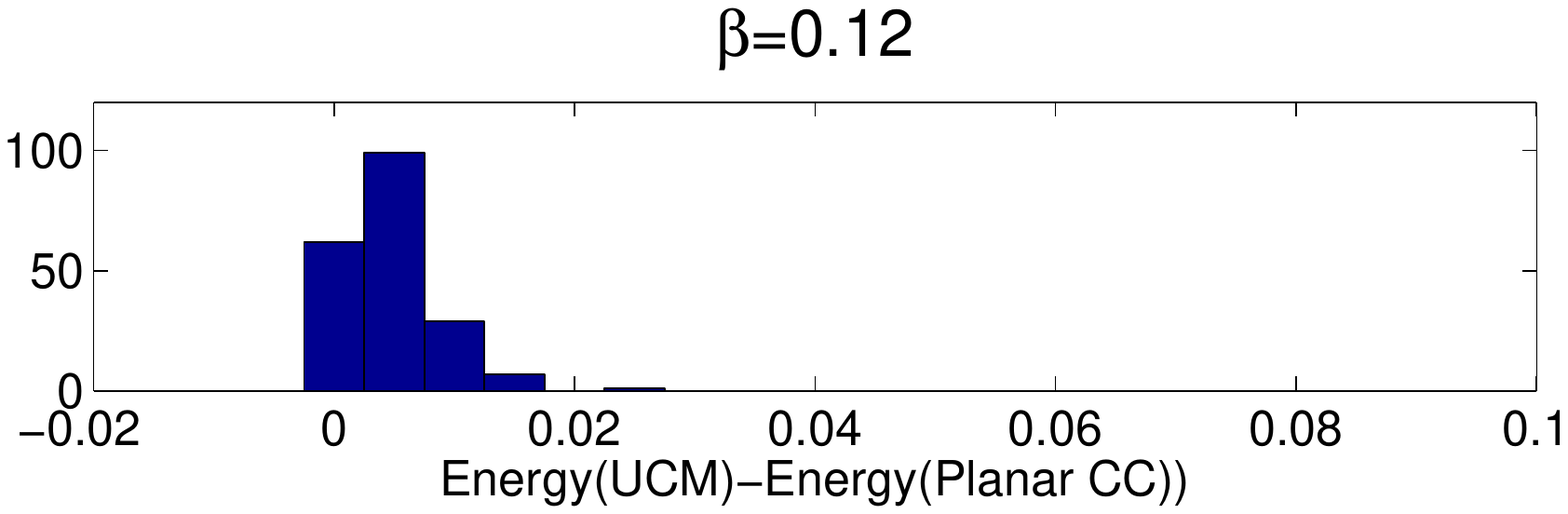}
\\
\caption{We find that our algorithm returns lower-energy segmentations than the
UCM algorithm. This suggests either a mismatch between the correlation clustering
model and the ground-truth or that our model is using suboptimal settings for
the local boundary cues \label{fig:comparison2}}
\end{figure}

We have presented a novel, fast algorithm for finding high quality correlation
clusterings in planar graphs. Our algorithm appears to outperform existing
approaches on a variety of real problem instances.  Our method exploits
decomposition into subproblems that lack efficient combinatorial algorithms but
are still tractable in the sense having efficient oracles. This
offers a new technique in the toolkit of Lagrangian relaxations that we expect
will find further use in the application of dual-decomposition to vision problems.

\noindent {\bf Acknowledgments:} This work was supported by a Google Research
Award, the UC Lab Fees Research Program and NSF DBI-1053036.

\bibliographystyle{splncs}
\bibliography{pc}
\newpage
\section*{Appendix A: Derivation of the LP dual}
It is informative to analyze the dual LP to the bound optimization presented in
Equation (14) of the paper:
\begin{align}
&\max_{\lambda} \sum_e (\theta_e - \lambda_e) \\
&\text{s.t.} \quad \theta_e \leq \lambda_e \leq \max\{0,\theta_e\} \notag \\
&\quad \quad \sum_{e} \lambda_e X_e \geq 0 \quad \quad \forall X \in \mathcal{C}_2 \notag
\end{align}
In order to write this in a standard form, define
\[
  \tilde{\lambda_e} = \lambda_e - \theta_e
\]
and
\[
  \tilde{\theta_e} = \max(0,\theta_e)-\theta_e = -\min(0,\theta_e)
\]
Let $C$ be a matrix whose rows are the collection of cut indicator
vectors $X$. Let us assume for now that $C$ contains the entire set of cut
vectors.  We can write the LP in standard form:
\begin{align}
&\max_{\lambda} -1^T \tilde{\lambda} \\
&\text{s.t.} \quad \tilde{\lambda} \geq 0 \notag\\
&\quad \quad \tilde{\lambda} \leq \tilde{\theta} \notag \\
&\quad \quad -C \tilde{\lambda} \leq C \theta \notag 
\end{align}
which has the following dual LP:
\begin{align}
&\min_{\alpha,\beta} \theta^T C^T\alpha + \tilde{\theta}^T\beta \\
&\text{s.t.} \quad \alpha \geq 0 \notag\\
&\quad \quad \beta \geq \max(0,C^T\alpha - 1) \notag 
\end{align}
To further simplify the expression, we define the set ${\mathcal C}_2^\triangle =
\{C^T\alpha,
\alpha\geq 0\}$ which is the convex cone known as the ``cut cone''\cite{Deza}.
Since $\tilde{\theta}\geq0$, $\beta$ will always take on its minimum allowable value
so we can collapse it into the objective, yielding:
\begin{align}
&\min_{z} \theta^T z - \min(\theta^T,0)\max(z-1,0)\\
&\text{s.t.} \quad z \in {\mathcal C}_2^\triangle \notag
\end{align}
Observe that the second term in the objective is $0$ when $z\leq1$ and is
positive when $z>1$.  

If we drop the upper-bound constraints on $\lambda$ in Equation (14) in the main
paper, we get the simplified dual LP:
\begin{align}
&\min_{z} \theta^T z \\
&\text{s.t.} \quad z \in {\mathcal C}_2^\triangle \notag\\
&\quad \quad z \leq 1 \notag
\end{align}

\noindent As discussed in the paper, this analysis provides several
insights on the nature of the lower-bound:

\begin{enumerate}
\item{One can see that geometrically, the PlanarCC bound is equivalent to
optimizing over a relaxation of the multi-cut polytope given by the intersection
of the cut cone and the unit hypercube.  This seems natural since the cut cone
and the multi-cut cone coincide~\cite{Deza} and the cut cone can be compactly
described for planar graphs.}\\

\item{Intuitively, the upper-bound constraints on $\lambda$ in the original LP
are irrelevant to the final value of the LP in the case that all cuts are
included in $C$. This is because any multicut can be represented as a sum of
isolating cuts, each with weight $0.5$. If such a solution optimizes LP (5), it
is also optimal for LP (4). We give a detailed proof in the next section which 
holds even when $C$ only includes a subset of cuts.}\\

\item{Finally, one can see the relation between this bound and the standard
cutting plane approaches used, for example, in \cite{Sontag07} or the ILP solution of
\cite{Andres11}.  In a standard cutting-plane approach, one optimizes the LP
relaxation $\min \theta^T z$ with a subset of the constraints that define the
multi-cut polytope and then successively add constraints, carving away parts of
the search space until an integral solution is found.  In contrast, each time
we add a cutting plane to our primal LP, this adds another row to the matrix C
in the dual LP which expands the set of allowable solutions $z$. Thus our
algorithm can be viewed as a delayed column generation scheme for the dual LP
in which we keep growing the space of reachable $z$ until an optimum is found.}
\end{enumerate}

\section*{Appendix B: Additional constraints on $\lambda$ don't affect the bound}

In Section 5 we introduced the constraint $\lambda_e \leq \max \{
0,\theta_e\}$ without a formal justification.  Here we show this constraint
does not decrease the lower-bound.  Suppose we first optimize the lower-bound
without including the constraint $\lambda_e \leq \max \{ 0,\theta_e\}$. Let
$\lambda^*$ denote the optimizing parameters.  Our strategy is to show that
$\lambda^*$ can be modified to satisfy $\lambda_e \leq \max \{ 0,\theta_e\}  \;
\forall e$ without loosening the bound.  

We first restate the definition of the lower-bound:
\begin{align}
\label{equallist}
CC^* \geq \sum_{e}(\theta_e-\lambda^*_e)+\min_{X\in \mathcal{C}_2}\sum_e \lambda^*_e X_e
\end{align}
The convex hull of the set of cut indicator vectors is known as the cut
polytope~\cite{Deza} which we denote $\mathcal{C}^\square_2$.  For planar
graphs, this polytope is compactly described by the set of cycle inequalities
\cite{BarahonaMahjoub86}.  One can
relax the discrete optimization over $\mathcal{C}_2$ to an LP over the cut
polytope (see e.g., \cite{Sontag07}).  In the following analysis, we work with
the dual of this relaxed LP in which we have a collection of dual variables
$\{\phi_e^c\}$ corresponding to the constraint on each cycle $c$ of the planar
graph~\cite{YarkonyThesis}.  
\begin{align}
\min_{X\in \mathcal{C}^\square_2}\sum_e \lambda^*_e X_e
&=\max_{\phi: \sum_c \phi^c_e=\lambda^*_e}\sum_c \min_{X^c \in \mathcal{C}_2}\sum_e \phi^c_e X^c_e
\end{align}
The right-hand side corresponds to dual-decomposition into a
collection of subproblems, each of which is a cycle from the original graph.
Let $\phi^\star$ denote an optimizer of this dual LP.  We write our
lower-bound in terms of this cycle decomposition as: 
\begin{align}
\label{cyclereplace}
CC^*\geq \sum_e( \theta_e-\lambda^*_e) +\sum_c \min_{X^c\in \mathcal{C}_2}\sum_e \phi^{*c}_{e}X^c_e
\end{align}

\noindent\textbf{Lemma 1:} For all cycles $c$, $\min_{X^c\in \mathcal{C}_2}\sum_e \phi^{*c}_{e}X^c_e=0$.\\

Lemma 1 holds because the minimum energy of each subproblem for a cycle $c$ is
upper-bounded by zero and the sum of all the cycle subproblem energies is zero
due to the constraint $\lambda^\star \in \Omega$.  Notice that if there is a negative
valued parameter $\phi^{*c}_e$ on the $c$'th cycle sub-problem then any other
edge $f\neq e$ must have a parameter setting such that
$\phi^{*c}_e+\phi^{*c}_f\geq 0$.  Otherwise, the configuration which cuts $e$
and $f$ would have negative energy. In particular, this implies that each
cycle subproblem can only have a single negative parameter.\\

\noindent \textbf{Lemma 2:  } For each edge $e$ with $(\theta_e-\lambda^*_e<0)$
contained in a cycle $c$, either ($\phi^{*c}_e<0$) or ($\exists f \; \;
\mbox{s.t.} \; \; \phi^{*c}_e+\phi^{*c}_f=0$). \\

Suppose there existed an edge $e$ and cycle $c$ for which the implication of the
lemma is false, that is ($\theta_e-\lambda^*_e<0$), ($\phi^{*c}_e\geq0$) and
($\forall f : \phi^{*c}_e+\phi^{*c}_f >0$).  This would mean there is no
minimizing configuration of the cycle subproblem that includes edge $e$ (such a
cut would have positive weight). However, edge $e$ is necessarily cut in the
single edge problem. In such a case the lower-bound could be tightened by the
following update:
\begin{align}
& f=\mbox{arg} \min_{f\neq e}(\phi^{*c}_e+\phi^{*c}_f).\\
\nonumber & V=\min[-(\theta_e-\lambda^*_e),\phi^{*c}_e+\phi^{*c}_{f}]\\
\nonumber & \lambda^*_e\Leftarrow \lambda^*_e- V\\
\nonumber & \phi^{*c}_e \Leftarrow \phi^{*c}_e-V
\end{align}
This update would drive up the energy of the single edge problem thus
increasing the lower-bound by a positive quantity $V$. Since the lower-bound
is tight by assumption, such an edge $e$ and cycle $c$ must not exist.\\

\noindent \textbf{Modifying $\lambda$ to satisfy the constraint:}
We now describe an iterative procedure that starts with $\lambda^*$ and $\phi^*$ and
produces a modified $\lambda^+$and $\phi^{+}$ obeying the constraint:
\begin{align}
\label{altparamconst}
(\theta_e-\lambda^+_e<0) \rightarrow (\phi^{+c}_e\leq 0) \quad \forall[c,e]
\end{align}
The lower-bound corresponding to $\lambda^+$ will have the same value as that
of $\lambda^*$ and will satisfy the additional upper-bound constraint as desired.

For each $(e,c)$ such that ($\theta_e-\lambda^*_e<0$) and ($\phi^{*c}_e> 0$),
Lemma 2 establishes that there exists an edge $f$ so that
$(\phi^{*c}_e+\phi^{*c}_f=0$).  Choose one such edge $f$ and apply the
parameter updates:
\begin{align}
\label{modparams2}
& V\Leftarrow \max[\theta_e-\lambda^*_e,-\phi^{*c}_{e}]\\
\nonumber & \phi^{*c}_{e} \Leftarrow \phi^{*c}_{e}+V\\
\nonumber & \phi^{*c}_f \Leftarrow \phi^{*c}_{f}-V\\
\nonumber & \lambda^*_e \Leftarrow \lambda^*_e+V \\
\nonumber & \lambda^*_f \Leftarrow \lambda^*_f-V
\end{align}
Repeatedly apply these updates until there exist no $(e,c)$ such that
($\theta_e-\lambda^*_e<0$) and ($\phi^{*c}_e> 0$).  These updates do not change
the minimizing configuration or energy of either the cycle or edge subproblems.
They also respect the lower-bound constraint $\theta_e \leq \lambda_e$.  Thus
the bound remains constant.  The final results of this procedure are
denoted $\lambda^+$ and $\phi^{+}$. \\

\noindent \textbf{Lemma 3:} For all edges $e$, $(\theta_e-\lambda^+_e<0
)\rightarrow ( \lambda^+_e \leq 0)$\\ 

The algorithm terminates when $(\theta_e-\lambda^+_e<0) \rightarrow
(\phi^{+c}_e\leq 0)$ for all cycles $c$ and edges $e$.  Since $\phi^+$ is a
reparameterization of $\lambda^+$ we have that $\sum_c \phi^{+c}_e=
\lambda^+_e$ for each edge $e$ which establishes the lemma.\\

\noindent \textbf{Claim}: For all edges $e$, $\lambda^+_e \leq \max \{ 0,\theta_e\}$ \\

If ($\theta_e-\lambda^+_e=0$) then the claim is satisfied.  If
$(\theta_e-\lambda^+_e<0)$ then by Lemma 3 we have ($\lambda^+_e \leq 0$).
For such an edge, it must be that $\theta_e < 0$ as we can't simultaneously
have $(\theta_e-\lambda^+_e<0)$, $\lambda^+_e \leq 0$ and $\theta_e \geq 0$.
Thus, we can transform any optimizer $\lambda^\star$ into an optimizer
$\lambda^+$ that achieves the same lower-bound and satisfies the additional
constraints.
\end{document}